\documentclass[letterpaper]{article} % DO NOT CHANGE THIS
\usepackage{aaai2027}  % DO NOT CHANGE THIS
\usepackage[hyphens]{url}  % DO NOT CHANGE THIS
\usepackage{graphicx} % DO NOT CHANGE THIS
\usepackage{natbib}  % DO NOT CHANGE THIS AND DO NOT ADD ANY OPTIONS TO IT
\usepackage{caption} % DO NOT CHANGE THIS AND DO NOT ADD ANY OPTIONS TO IT
\usepackage{algorithm}
\usepackage{algorithmic}

\usepackage{newfloat}
\usepackage{listings}
\DeclareCaptionStyle{ruled}{labelfont=normalfont,labelsep=colon,strut=off} % DO NOT CHANGE THIS
\floatstyle{ruled}
\newfloat{listing}{tb}{lst}{}
\floatname{listing}{Listing}

\usepackage{booktabs}

\usepackage{amsmath, amsthm, amsfonts}

\newtheorem{assumption}{Assumption}
\newtheorem{definition}{Definition}
\newtheorem{theorem}{Theorem}

\usepackage[capitalize]{cleveref}
\crefname{section}{Sec.}{Secs.}
\Crefname{section}{Section}{Sections}
\crefname{subsection}{Sec.}{Secs.}
\Crefname{subsection}{Section}{Sections}
\Crefname{table}{Table}{Tables}
\crefname{table}{Tab.}{Tabs.}
\Crefname{figure}{Figure}{Figures}
\crefname{figure}{Fig.}{Figs.}

\usepackage{pifont}

\newcommand{\cmark}{\ding{51}}% ✓
\newcommand{\xmark}{\ding{55}}% ✗

\usepackage{multirow}

\title{Collapse of Patches: Ranking Image Patches for Efficient Visual Modeling}
\author{
    Wei Guo,
    Shunqi Mao,
    Zhuonan Liang,
    Xuanhua Yin,
    Heng Wang,
    Weidong Cai\corresponding
}
\affiliations{
    The University of Sydney\\
    \{wei.guo, shunqi.mao, zhuonan.liang, xuanhua.yin, heng.wang, tom.cai\}@sydney.edu.au
}
\nocopyright

\begin{document}

\maketitle

\begin{abstract}
Observing certain patches in an image reduces the uncertainty of others. Their realization lowers the distribution entropy of each remaining patch feature, analogous to collapsing a particle's wave function in quantum mechanics. This phenomenon can intuitively be called patch collapse. To identify which patches are most relied on during a target region's collapse, we learn an autoencoder that softly selects a subset of informative patches during reconstruction. Graphing these learned dependencies for each patch's PageRank score reveals the optimal patch order to realize an image. We show that respecting this order benefits various masked image modeling methods. First, autoregressive image generation can be boosted by finetuning with the ordered generation sequence. Second, we introduce a new setup for image classification by exposing Vision Transformers only to high-rank patches in the collapse order. Seeing 22\% of such patches is sufficient to achieve high accuracy. With these experiments, we propose patch collapse as a novel image modeling perspective that promotes vision efficiency.
\end{abstract}


\section{Introduction}
\label{sec:intro}

Images are structures of mutual dependence: observing parts of an image often reveals information about others \cite{ImageStats}. The observation (understanding), or realization (generation), process of an image, when modeled along a sequence of local patch regions, is then analogous to the \textbf{collapse of a wave function} in quantum mechanics \cite{feynman}. Once a particular patch is measured, the remaining patches' uncertainty reduces around the observed evidence. Intuitively, we refer to this image uncertainty reduction process as the \textbf{collapse of patches}.

\begin{figure}
    \centering
    \includegraphics[width=\linewidth]{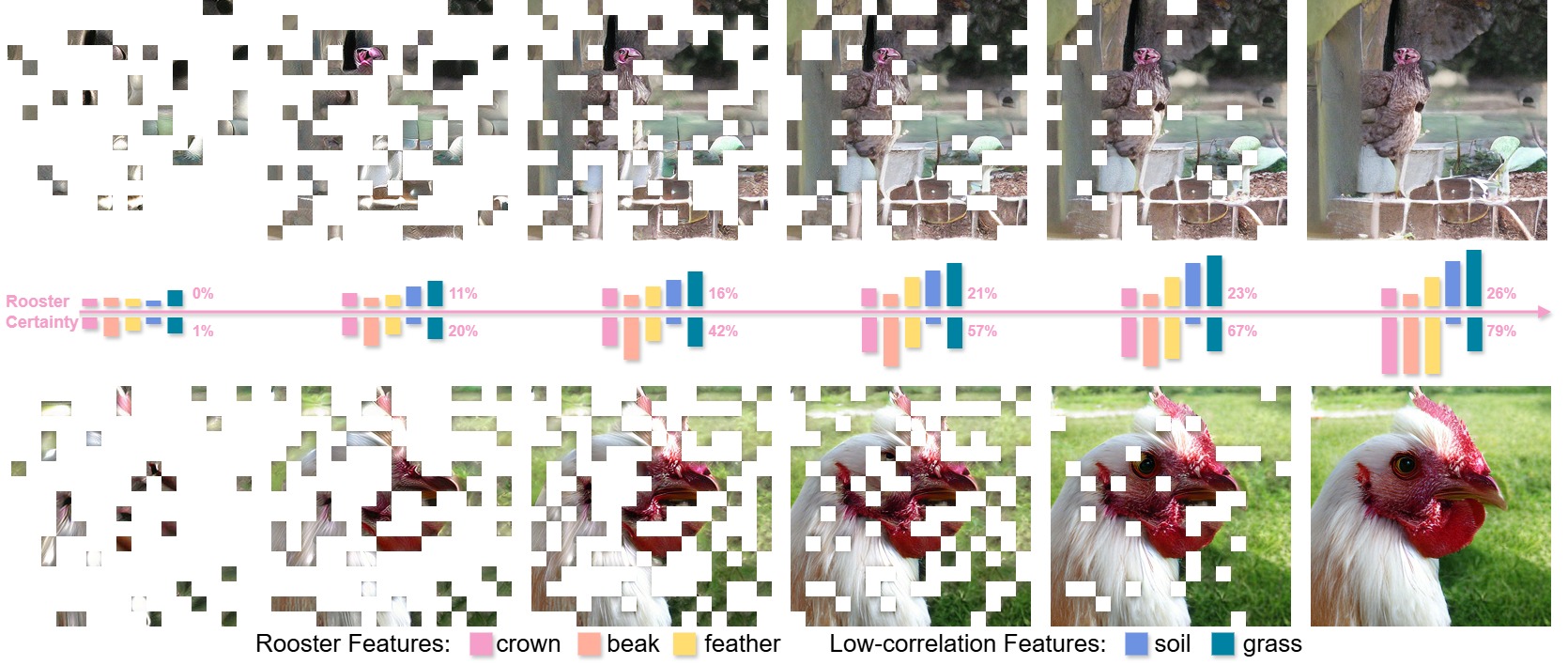}
    \caption{\textbf{Patch synthesis in random and collapse orders.} We autoregressively generate rooster image patches following random order (above) and collapse order (below). The latter synthesizes prominent rooster features and reduces image uncertainty more effectively.}
    \label{fig:teaser}
\end{figure}

Different patches collapse the uncertainty of the residual image with different effectiveness, which is a nontrivial phenomenon. Seeing the beak of a rooster first, for instance, constrains the leftover image information more than seeing a background field, as shown by the comparison of uncertainty reduction effects from an autoregressive (AR) image generator \cite{mar} following different patch synthesis orders in \cref{fig:teaser}. In reality, a human painter also starts with a sketch of their subject's important parts in order to neatly constrain the underlying visual uncertainty \cite{draw}. Aside from image synthesis, the human ability to correlate partial visual contents differently in a scene is also essential for completing vision tasks with high efficiency, by quick scans over important image regions \cite{context}.

Most modern vision models, however, treat image patches as uniformly correlated samples in masked image modeling (MIM) by randomly masking them with uniform probability, \textit{e.g.}, in autoregressive generation \cite{PixelRNN, PixelCNN, VQ-GAN, MaskGIT, mar, VAR, xAR, HMAR} or masked classification \cite{MAE, SimMIM, MixMAE, CAPI, MIMIR, CMAE, SMAE}. This assumption overlooks the nonuniform perceptual importance of image patches, whereas biological visual systems preferentially allocate attention to salient locations as an inductive bias, enabling rapid and computationally efficient scene analysis \cite{saliency}.

In this work, we formalize the problem of patch collapse respecting its patch-wise priorities. We assume that when certain patches of an image are observed, the feature distributions of the unobserved patches shift from broad, high-entropy shapes to concentrated, low-entropy states. The further assumption is that an image's patches can be ranked based on their elicited shift strengths during this process. To study how this collapse unfolds, we train a \textbf{Collapse Masked Autoencoder (CoMAE)} whose encoder selectively masks image patches with noise injection, conditioning the decoder in reconstructing a target patch. We show that this encoder-predicted mask, as a normalized continuous vector, saliently weights each patch's collapse contribution. Our experiments validate the assumptions above: only a subset of patches are most responsible for each given patch's collapse, as polarized selection weights emerge to optimize reconstructions. Furthermore, we observe that this selection set varies across patches, suggesting that each patch has a distinct collapse dependency and contributes to the global image certainty with different effectiveness.

To analyze these patch dependencies at the image level, we map out a directed acyclic graph of patches with the patch-level CoMAE selection masks as edge weights. Applying PageRank \cite{PageRank} to this graph yields an ordering of patches by their collapse independence, defining a provably optimal uncertainty reduction sequence in which an image realizes itself. We term these sequences \textbf{collapse orders}. Our visualizations show that high-rank patches in the collapse order outlines important shapes in an image. Additionally, we observe that the inter-class collapse orders across ImageNet \cite{ImageNet} samples exhibit moderate similarity, which suggests their depiction of a consistent underlying structure behind different images.

We demonstrate that collapse order offers beneficial supervision to MIM methods. When integrated into an AR image generator MAR \cite{mar}, respecting the collapse order improves sample quality both quantitatively and qualitatively over the original model. Alongside generation, respecting the collapse order also leads to efficient image classification: by training on patches with high collapse priorities, a Vision Transformer (ViT) \cite{ViT} can maintain high accuracy while processing just $22\%$ of the image content. In contrast, conventional full-image classifiers sacrifice computation to redundant patches that contribute diminished marginal discriminative value.

In summary, our work has the following contributions:

\begin{enumerate}
    \item We introduce and formalize the problem of patch collapse, which offers a novel perspective to model image structures that are crucial for efficient visual systems.
    \item We present CoMAE, an effective method to learn the image-specific patch collapse order.
    \item By supervising with the learned collapse order, we improve MIM methods' performance in AR image generation and masked image classification. These experiments show the generalizable applicability of patch collapse modeling to different vision tasks.
\end{enumerate}

\section{Related Works}
\label{sec:related}

\paragraph{Stochastic Masked Image Modeling.}
Inspired by masked language modeling \cite{BERT}, masked image modeling (MIM) predicts corrupted local units from an image to learn generalizable visual features. Stochastic MIM (SMIM) methods reconstruct randomly masked image portions to obtain self-supervised visual features, assuming uniform patch correlations. The pioneering denoising autoencoder \cite{denoise-ae} explores robust feature learning through partial latent pattern corruption. Later, Context Encoder \cite{con-encoder} employs convolutional neural networks to inpaint stochastic image regions for representation learning. Masked Autoencoders (MAE) \cite{MAE} apply high random masking ratios (e.g., $75\%$) in an asymmetric transformer for scalable learning. SimMIM \cite{SimMIM} simplifies this process with larger patches and RGB regression. Painter \cite{Painter} expands MIM to various image-to-image mapping tasks. MixMAE \cite{MixMAE} incorporates image mixing alongside MIM for data augmentation. VideoMAE \cite{VideoMAE} extends MIM to consider spatio-temporal masking for videos, while OmniMAE \cite{OmniMAE} explores masked modeling with multimodal data. Recently, CAPI \cite{CAPI} enhances SMIM features by predicting missing patches w.r.t. an unmasked teacher. MIMIR \cite{MIMIR} improves SMIM's adversarial robustness via mutual information-based reconstruction. These SMIM methods learn generalizable visual features but ignore the variance of inter-patch dependencies, leading to inefficient representations. In contrast, our CoMAE adaptively masks trivial patches to model patch dependencies accurately, deriving stronger guidance for MIM methods.

\paragraph{Adaptive Masked Image Modeling.}
Adaptive MIM (AMIM) methods adjust image-wise masking during training to effectively target informative regions. CMAE \cite{CMAE} unifies contrastive learning with MIM for stronger representation disambiguation. SiamMAE \cite{SMAE} leverages asymmetric masking for video correspondences. AttMask \cite{AttMask} generates attention-based masks from a teacher model for AMIM. AdaMAE \cite{AdaMAE} masks visible video tokens adaptively with an additional sampling network. SemMAE \cite{SemMAE} uses semantic masks from a ViT to mask an image's dominant shapes. CL-MAE \cite{CLMAE} employs curriculum learning to adapt MIM to harder masks that hinder reconstruction. The recent RAM++ \cite{RAM++} uses adaptive masks for blind image restoration, while Self-Guided MAE \cite{SelfGuidedMAE} introduces self-guided informed masking based on early-stage patch clustering in MIM. Although AMIM methods efficiently adapt to data salience, they do not explicitly model image uncertainty reduction across patches. Our CoMAE formalizes patch collapse, identifying the global patch orders for image realization. CoMAE provides a unique perspective of image structures absent in prior methods, implemented with AMIM but not limited to its coverage in vision modeling.

\paragraph{Autoregressive Image Generation.}
Autoregressive (AR) models generate images sequentially with localized representations, often employing masking for unified learning. Classic AR methods, \textit{i.e.}, PixelRNN \cite{PixelRNN} and PixelCNN \cite{PixelCNN}, follow fixed raster orders during generation. VQ-GAN \cite{VQ-GAN} generates quantized image tokens in the same order.  MaskGIT \cite{MaskGIT} predicts quantized tokens with a scheduled stochastic order. MAGE \cite{MAGE} learns discrete image tokens and their generation in a unified manner. MAR \cite{mar} follows the same order but replace MaskGIT's discrete tokens with continuous latents. VAR \cite{VAR} employs next-scale prediction for AR. MAGVIT \cite{MAGVIT} extends AR generation to videos. Recently, HMAR \cite{HMAR} improves AR efficiency with hierarchical multi-step prediction. xAR \cite{xAR} generalizes next-token generation to flexible units such as cell or scales (next-x predictio). While different random or fixed orders are explored by these methods, our collapse order offers a data-salient strategy in modeling the generation priorities of AR units, yielding significant gains without drastic remodeling and retraining. Our strategy can also be integrated with various next-x methods straightforwardly.

\paragraph{Efficient Vision Transformers via Token Pruning.}
Token pruning methods enhance the computational efficiency of Vision Transformers (ViTs) by selectively removing redundant tokens. DynamicViT \cite{DynamicViT} estimates token importance scores to hierarchically prune trivial tokens. Similarly, ATS \cite{ATS} and A-ViT \cite{A-ViT} prune tokens with attention scores. AdaViT \cite{AdaViT} learns a decision policy model to drop various inference units. EViT \cite{EViT} fuse attention-pruned tokens into a single representation to retain more information. SPViT \cite{SPViT} employs an attention-based selector for this fusion. For semantic segmentation, DToP \cite{DToP} dynamically prunes easy tokens via early exiting based on confidence thresholds. While these methods investigates feature pruning across the model architecture, our collapse order directly operates on explicit image patches. We show that vision efficiency can be significantly boosted by pruning on the model-agnostic image space alone. This decoupling of data and model also suggests that our method can readily combine with token pruning techniques.

\begin{figure*}
    \centering
    \includegraphics[width=\linewidth]{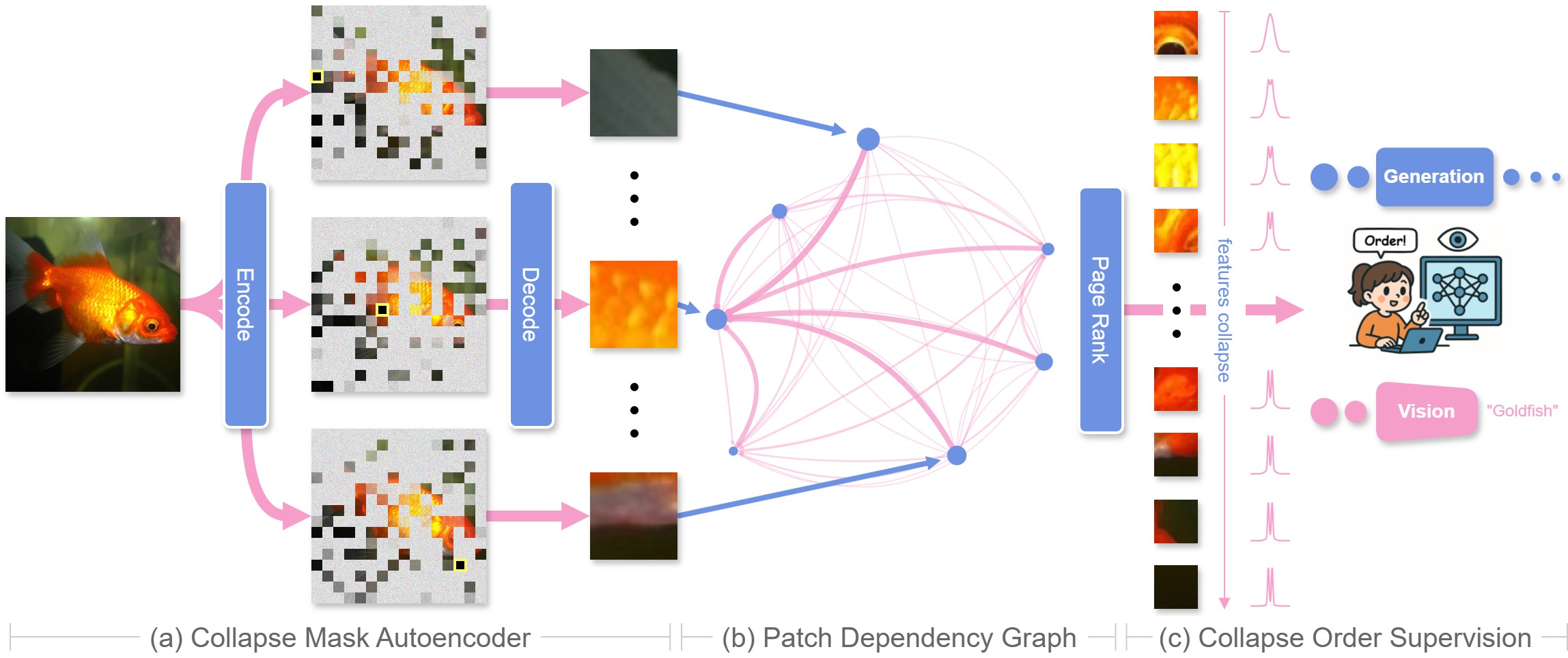}
    \caption{\textbf{Pipeline overview.} Given an image, the CoMAE encoder selects the most influential patches needed to reconstruct each patch, while trivial patches are masked with heavier noise injection. These selection weights form a patch dependency graph on which we compute the PageRank scores to determine the collapse order of patches, where higher-rank patches are less dependent on the rest of the image. Finally, we use this ranking to supervise image generation and classification tasks to follow the correct patch processing order.}
    \label{fig:pipe}
\end{figure*}

\section{Method}
\label{sec:method}

For modeling efficiency, we first pass an image through a variational autoencoder (VAE) \cite{ae-bayes, ae-gen, VAE} to represent $N$ patches with a set of embeddings as $\left\{\mathbf{e}_n\right\}^N$. Each patch's feature distribution is conditioned on other patches, which can be expressed by a probability distribution $P\left(\mathbf{e}_{n}\vert \left\{ \mathbf{e}_i \right\}_{i \ne n}^{N}\right)$. Our problem can then be formulated as finding a ranking function $R\left(\mathbf{e}_n\right)$ such that:

\begin{equation}
\label{eq:entropy-model}
    H_c = \sum_{i=1}^{N}H\left[P\left(\mathbf{e}_n\vert\left\{\mathbf{e}_i;R\left(\mathbf{e}_i\right)>R\left(\mathbf{e}_n\right)\right\}^N_{i\ne n}\right)\right],
\end{equation}

\noindent where $H$ measures the distribution entropy, is minimized.

To model $R$, we first learn each patch's dependencies on other patches with an autoencoder's encoded masking weights in \cref{subsec:comae} through reconstruction. A directed acyclic graph of patches can then be drawn in \cref{subsec:patch-rank}, where the edge weights are patch dependencies. Computing the PageRank \cite{PageRank} scores of this graph yields the output of $R$, which constitutes our collapse order as the optimal patch realization sequence to reduce image uncertainty.

With the collapse order identified, we propose to improve the existing stochastic AR image generator MAR \cite{mar} by respecting this image structure in \cref{subsec:cmar}. To accomplish this, we train an MAR model with collapse order guidance. Finally, we investigate our collapse modeling's effectiveness on image classification by training a Vision Transformer (ViT) \cite{ViT} with collapse-order masks in \cref{subsec:cvit}.

\subsection{Collapse Masked Autoencoder}
\label{subsec:comae}
Following the motivation that some patches are more responsible for a specific patch's collapse than others, we assume masking $K$ patches from $\left\{ \mathbf{e}_i \right\}_{i \ne n}^{N}$ leaves $P$'s shape approximately unchanged. To find these influential patches, we train a Collapse Masked Autoencoder (CoMAE) model as shown in \cref{fig:pipe} (a). During reconstruction, the encoder $f$ follows ViT \cite{ViT} to pool visual information from $\left\{\mathbf{e}_i\right\}^N_{i\ne n}$ with self-attention blocks. It predicts a soft weight vector $\mathbf{w} \in [0,1]^n$ for patch selection as:

\begin{equation}
    \mathbf{w}_n=f\left(\left\{\mathbf{e}_i\right\}^N_{i\ne n};\mathbf{q}_n\right),
\end{equation}

\noindent where $\mathbf{q}_n$ is a learned positional embedding to inform the encoder of the under-reconstruction patch location. We then mask each patch embedding by noise injection w.r.t. $\textbf{w}_n$ as:

\begin{equation}
    \mathbf{e}_i^m=\alpha_i\mathbf{e}_i+\left(1-\alpha_i\right)\mathcal{N}\left(0,\mathbf{I}\right), \ i\ne n,  
\end{equation}

\noindent where $\alpha$ is an exponential decay interpolant with a hyperparameter $\sigma$ for controlled steepness:

\begin{equation}
    \alpha_i = \exp\left(-\frac{\left(1-\mathbf{w}_n^i\right)^2}{2\sigma^2}\right).
\end{equation}

The decoder $g$ is another stack of ViT-like self-attention blocks that pools information from the selected patches to reconstruct the target patch $\mathbf{e}^\ast_n$:

\begin{equation}
    \mathbf{e}^\ast_{n}=g\left(\left\{\mathbf{e}^m_i\right\}_{i\ne n}^N;\mathbf{q}_n\right).
\end{equation}

\noindent The patch reconstruction loss is simply $\mathcal{L}_r = \left\Vert \mathbf{e}_n - \mathbf{e}_n^\ast \right\Vert_1$. We alternatively optimize $f$ and $g$ with $\mathcal{L}_r$ in a batch-wise manner during training. Please refer to Supplementary section 3 for additional architecture details of CoMAE.

\paragraph{Polarization.}
To test our assumptions in \cref{sec:intro} that there exists a priority ranking of patches in patch collapse, we observe the training metrics of CoMAE in \cref{subsec:comae-props}. It can be seen that $\mathbf{w}$ polarizes to 0 and 1 as the reconstruction loss $\mathcal{L}_r$ converges to minima instead of staying uniform. This effect confirms that different patches contribute to an image's uncertainty reduction with different effectiveness, since only a subset of patches significantly contribute to each target patch's feature collapse.

\paragraph{Contrastive Regularization.}
Do different patches rely on different subset selections for their feature collapse? To answer this further question, we inquisitively add a contrastive objective to encourage the diversity of $\mathbf{w}$ as:

\begin{equation}
    \mathcal{L}_c=\frac{1}{N}\sum_{i=1}^{N}-\log\frac{\exp\left(\text{sim} \left(\mathbf{w}_i, \mathbf{w}_i\right)/\tau\right)}{\sum_{j=1}^{N}\exp\left(\text{sim} \left(\mathbf{w}_i, \mathbf{w}_j\right) /\tau\right)},
\end{equation}

\noindent where $\tau$ is a learnable temperature and $\text{sim}\left(\cdot,\cdot\right)$ measures the cosine similarity between two vectors. We then define the total loss to be $\mathcal{L}=\mathcal{L}_r+0.01\mathcal{L}_c$ and retrain CoMAE.

This ablation, detailed in \cref{subsec:comae-props}, shows that CoMAE plateaus at a significantly higher $\mathcal{L}_r$ without $\mathcal{L}_c$, assigning similar $\mathbf{w}$ to all patches. With $\mathcal{L}_c$, $\mathcal{L}_r$ is significantly reduced by diverse masks across patches and escapes local minima. Thus, it's reasonable to assume that the one-to-many dependency of each patch during collapse is diverse.

\subsection{Patch Ranking}
\label{subsec:patch-rank}
$N$ instances of $\mathbf{w}$ can be encoded for all patches in an image. Together they form a patch dependency graph with a $N\times N$ adjacency matrix $\mathbf{A}$ where $\mathbf{A}_{ij}=\mathbf{w}_{ij}$. To rank the independency of patches, we compute their PageRank scores from $\mathbf{A}$. A patch with high PageRank score has more influence on other patches. Please see Supplementary section 1 for a formal proof linking this ranking mechanism to the objective optimal dependency ranking $R$.

\begin{figure}
    \centering
    \includegraphics[width=\linewidth]{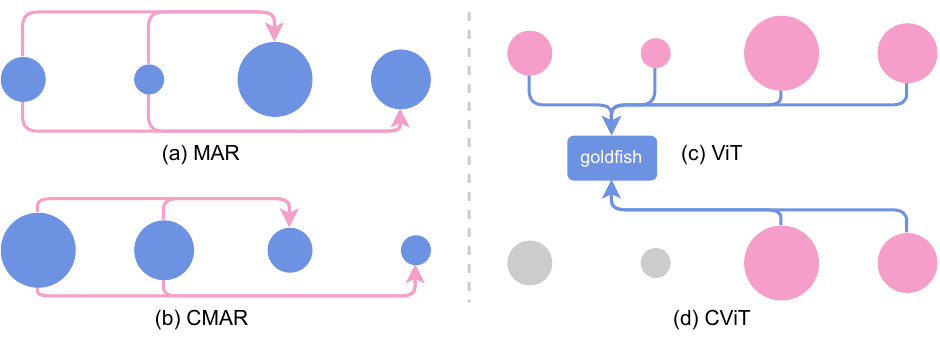}
    \caption{\textbf{Comparison of MIM methods.} Our generator (CMAR) and classifier (CViT) respect the collapse order.}
    \label{fig:tasks}
\end{figure}

\subsection{Collapsing Autoregressive Image Generation}
\label{subsec:cmar}

As shown in \cref{fig:tasks} (a), autoregressive (AR) image generators often follow random orders to generate patches sequentially. This stochastic arrangement assumes that patch dependencies follow a uniform distribution in an image. We apply our patch sequencing learned from CoMAE as extra supervision to retrain such a stochastic AR model, MAR \cite{mar}, which we call \textbf{Collapsed Mask Autoregressive Model (CMAR)} in \cref{fig:tasks} (b). CMAR learns to generate patches following our identified collapse order.

For each training sample, we first pass it through the CoMAE encoder and obtain its patch ranks. We then mask a random amount of low-rank patches and learn to generate them from high-rank ones. Since this sampling order of patches is deterministic, CMAR is more likely to overfit than the original model. To compensate for this effect, we keep $10\%$ sampling sequences with random ranks.

\subsection{Collapsing Image Classification}
\label{subsec:cvit}

Conventionally, an image classifier has access to the entire input image as in \cref{fig:tasks} (c). Since we show earlier that the realization of an image follows collapse order, it's intuitive to ask if such classifiers can maintain accurate with only high-rank patches in the collapse order.

Correspondingly, we train a ViT classifier on ImageNet \cite{ImageNet} under two settings: full-image and masked. We mask $0\sim99\%$ low-rank patches with a cosine annealing schedule favoring lower mask rates, which results in a \textbf{Collapsed Vision Transformer (CViT)} depicted in \cref{fig:tasks} (d). Instead of replacing the masked patches with mask tokens, CViT directly drops them from the input sequence.

\begin{figure*}
    \centering
    \begin{minipage}{0.49\textwidth}
        \centering
        \includegraphics[width=\textwidth]{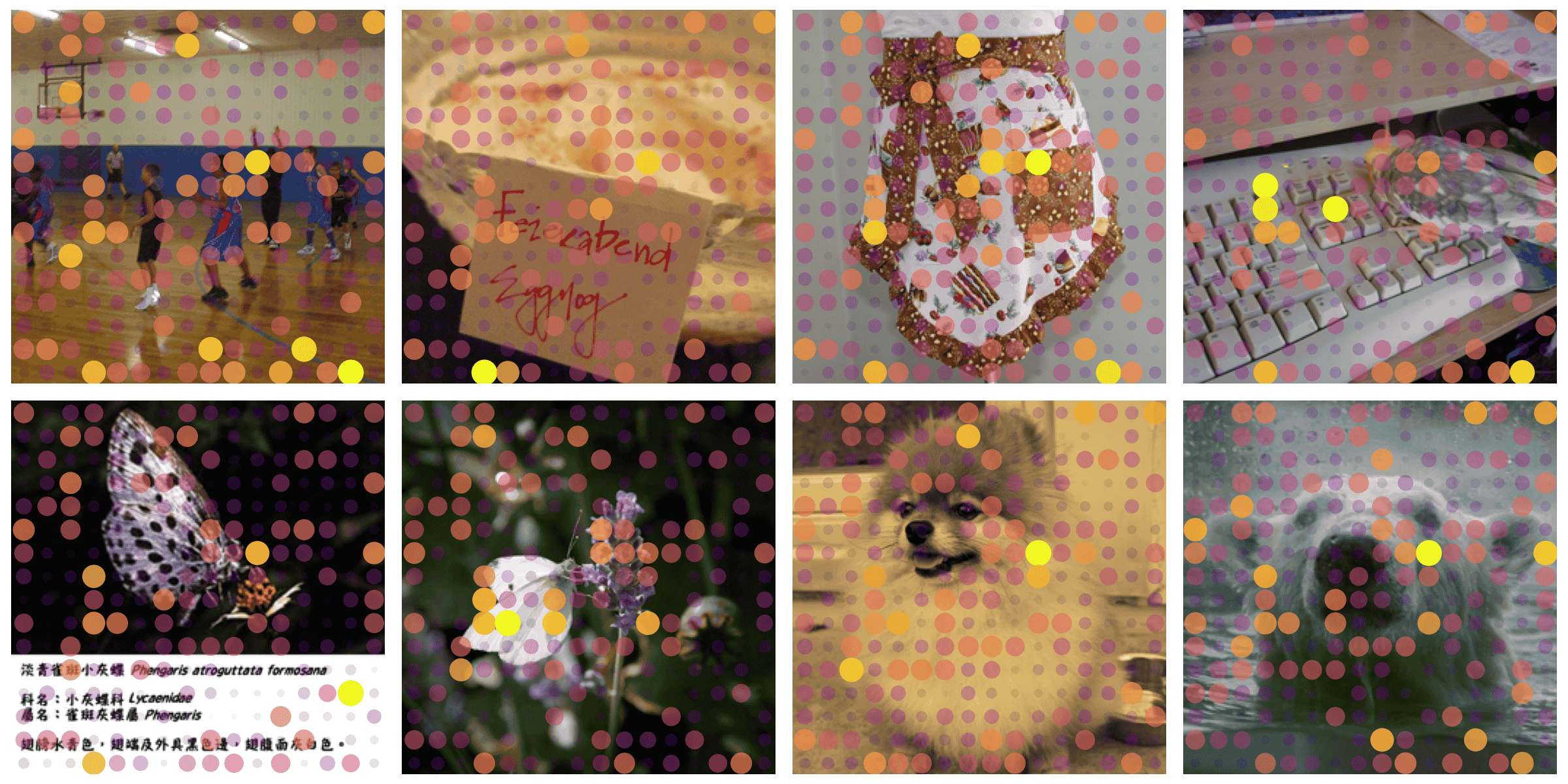}
    \end{minipage}
    \begin{minipage}{0.49\textwidth}
        \centering
        \includegraphics[width=\textwidth]{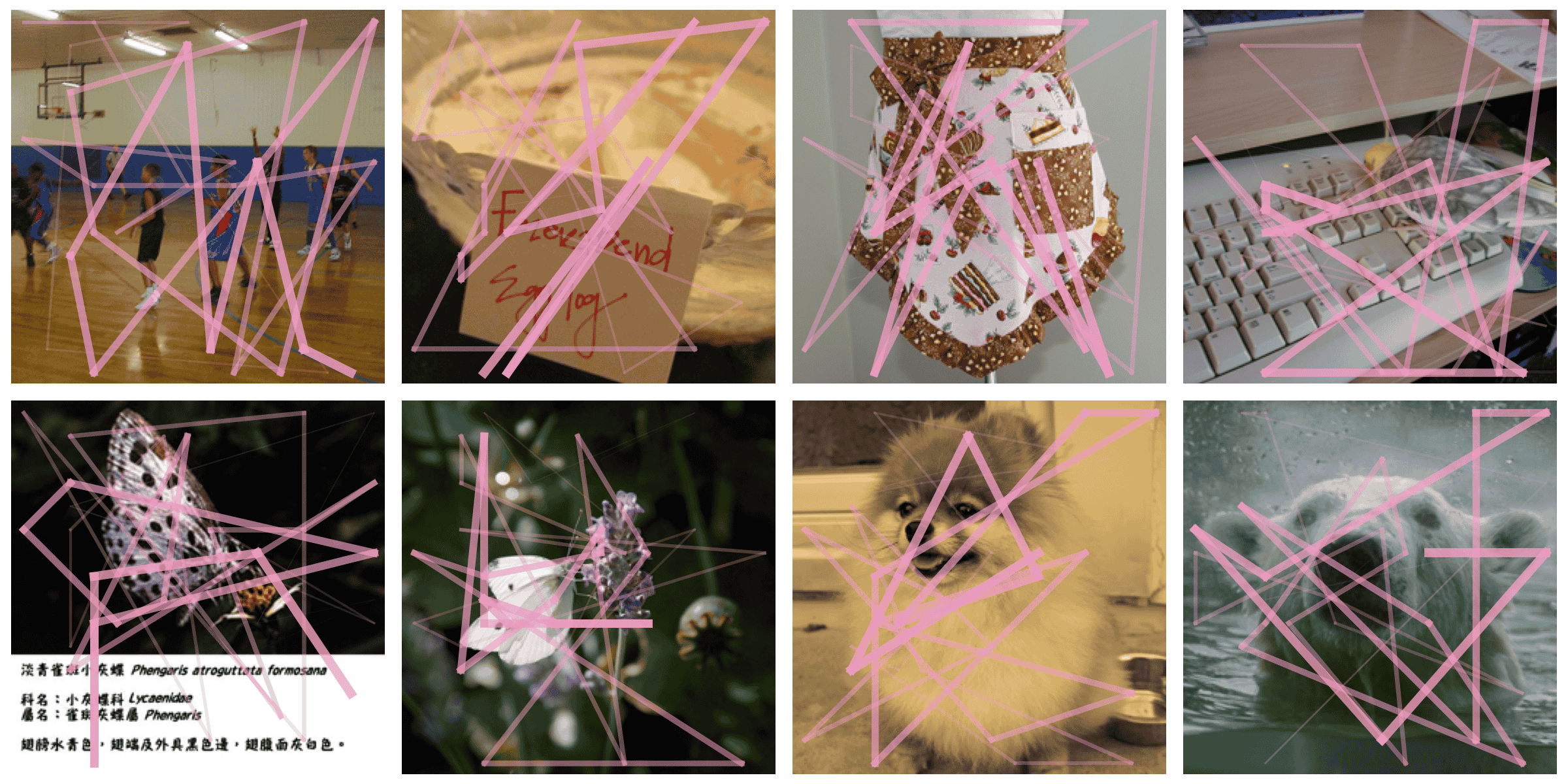}
    \end{minipage}
    \caption{\textbf{Visualization of collapse order.} The left figure shows image patches with different collapse ranks indicated by circle sizes. The right figure connects the top-ranked 64 patches by collapse order. Top patches outline important shapes in each image.}
    \label{fig:comae-qual}
\end{figure*}

\begin{table}
    \centering
    \begin{tabular}{c|c|c}
        \toprule
        Contrastive & Mask Entropy & Reconstruction Loss \\
        \midrule
        \xmark & 4.816 & 8.392\\
        \cmark & \textbf{4.267} & \textbf{1.567}\\
        \bottomrule
    \end{tabular}
    \caption{\textbf{Ablation of CoMAE's contrastive regularization.} Contrastive loss aids reconstruction and mask polarization.}
    \label{tab:comae-contrast}
\end{table}

\section{Experiments and Results}
\label{sec:experiment}

We conduct all our experiments on center-cropped $256\times256$ images from ImageNet-1k \cite{ImageNet}. These images are first processed by a KL-16 VAE \cite{LDM} used by MAR \cite{mar} into 256 16-dim tokens. We use a RTX5090 GPU for all trainings.

\paragraph{Implementation of CoMAE.}
Both the encoder and decoder of CoMAE have 12 attention blocks following ViT \cite{ViT}, with embedding dimensions 64 and 256 respectively. A four-layer MLP is appended to the decoder to output the final 16-dim target patch token. The encoder and decoder are alternatively optimized in a batch-wise manner. We train CoMAE for 512 epochs with a cosine annealing schedule decaying learning rate from 1e-4 to 0.

\paragraph{Implementation of CMAR.}
Training MAR from scratch is computationally daunting for our resources. Instead, we treat the pretrained MAR as an order-agnostic prior and fine-tune it for 24,000 steps (batch size 32) on the same ImageNet data with our collapse order. Due to restricted computational resources, we only experiment on the MAR-B variant. We linearly warmp up the learning rate to 1e-7 during the first $10\%$ steps and then decay it to 0 with cosine annealing. The model weights are updated with a per-step estimated mean average (EMA) of rate 0.99999. During inference, we set CMAR's classifier-free guidance (CFG) \cite{CFG} scale to 3.0 from our ablations in \cref{tab:cmar-cfg}. We keep the original MAR at CFG 2.9 for its optimal performance.

\paragraph{Implementation of CViT.}
We fine-tune the ImageNet-21k pretrained ViT-Base \cite{ViT} model on ImageNet-1k for classification of 1000 image classes. The training expands 3 epochs with a cosine decay learning rate shrinking from 1e-4 to 0. The model weights are updated by per-step EMA with rate 0.9999. We apply a random horizontal flip of training images for data augmentation.

\subsection{Properties of CoMAE}
\label{subsec:comae-props}

We provide experiments to corroborate our statements of CoMAE's behaviors. First, we show that CoMAE emerges polarized selections instead of uniformly distributing the weights during reconstruction optimization. To quantify polarization, we define a Mask Entropy metric as:

\begin{equation}
    H_\text{mask}=-\frac{1}{M}\sum^M_{i=1}\sum^N_{j=1}\mathbf{w}^i_j\log\mathbf{w}^i_j,
\end{equation}

\noindent where $M$ is the number of samples and $N$ is the number of patches. A lower $H_\text{mask}$ reflects more polarized mask distribution in $\mathbf{w}$. We show in Supplementary section 4 that $H_\text{mask}$ shrinks as reconstruction loss optimizes, indicating that only a subset of patches is responsible for a target patch's collapse.

Next, we observe that different target patches depend on different patch subsets for their collapse. Our contrastive ablations in \cref{tab:comae-contrast} reveal that this diverse selection further polarizes masks and significantly optimizes reconstruction. Additionally, we provide a visualization of the collapse orders illustrated in \cref{fig:comae-qual}. The patches outlining major objects have higher ranks in each image, which aligns intuitively with human's visual scanning order \cite{scan}. This emergent similarity between image synthesis and recognition is particularly intriguing, as it suggests a convergence of optimal scanning order between these opposite tasks.

\begin{figure*}
    \centering
    \includegraphics[width=\linewidth]{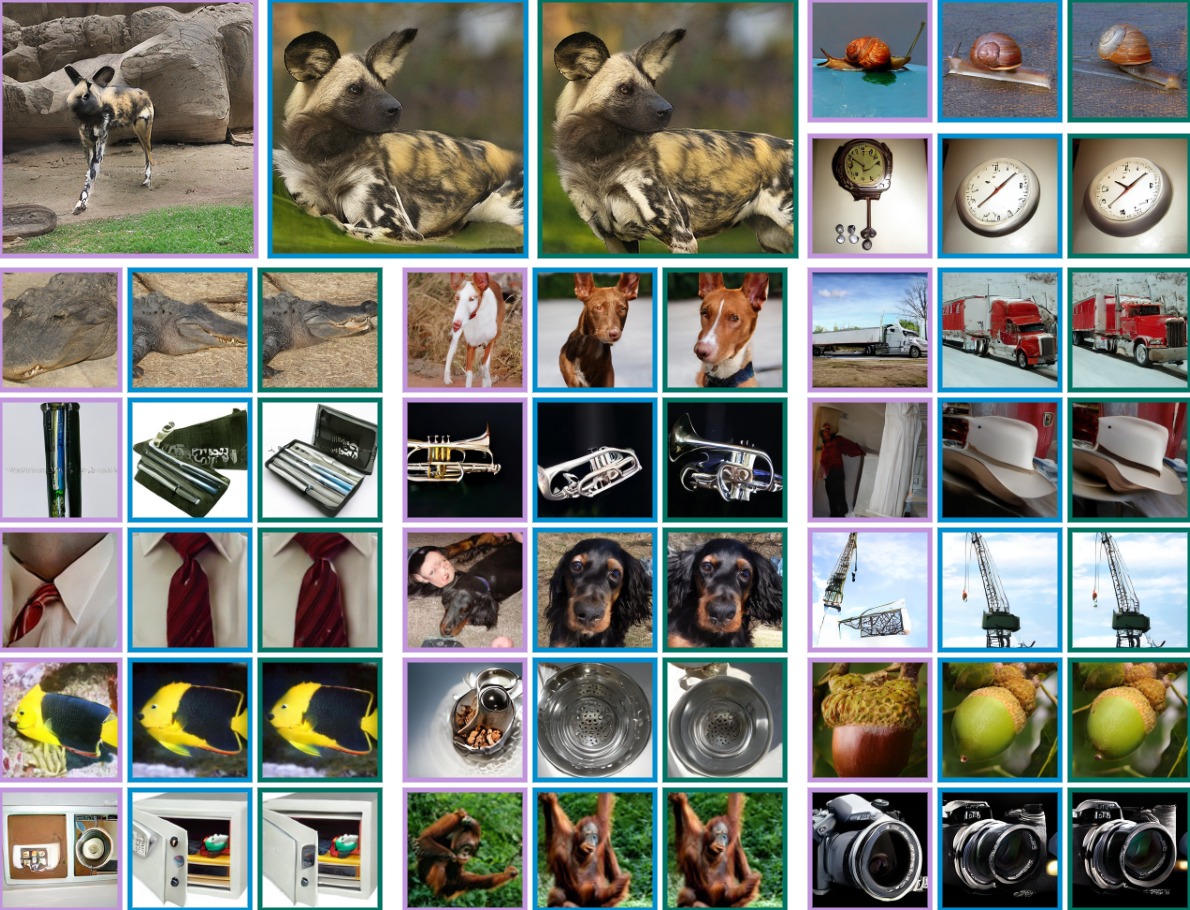}
    \caption{\textbf{Qualitative comparison of ARs.} MAR is framed in purple. Our MAR+C and CMAR results are in blue and green respectively.}
    \label{fig:cmar-qual}
\end{figure*}

\begin{table*}
    \centering
    \begin{tabular}{l|l|l|l|l|l|l|l|l|l|l}
         \toprule
         \multirow{2}{*}{Method} & \multicolumn{5}{c|}{w/o CFG} & \multicolumn{5}{c}{w/ CFG}\\
         & FID$\downarrow$ & tFID$\downarrow$ & IS$\uparrow$ & Pre.$\uparrow$ & Rec.$\uparrow$ & FID$\downarrow$ & tFID$\downarrow$ & IS$\uparrow$ & Pre.$\uparrow$ & Rec.$\uparrow$\\
         \midrule
         MAR & \textbf{7.114} & \textbf{3.498} & \textbf{194.50} & 0.784 & 0.571 & 5.997 & 2.330 & 281.48 & 0.822 & 0.571\\
         \midrule
         MAR+C (Ours) & 7.173 & 3.563 & 193.80 & \textbf{0.788} & 0.568 & 5.956 & 2.321 & \textbf{284.78} & \textbf{0.826} & 0.566\\
         CMAR (Ours) & 7.213 & 3.600 & 190.92 & 0.781 & \textbf{0.572} & \textbf{5.928} & \textbf{2.238} & 280.55 & 0.818 & \textbf{0.576}\\
         \bottomrule
    \end{tabular}
    \caption{\textbf{Generation performance of different AR methods.} The first place is \textbf{bolded}. MAR+C denotes the unfine-tuned MAR results following our collapse order. CMAR tests MAR fine-tuned with collapse order.}
    \label{tab:ar-comp}
\end{table*}

\begin{figure*}
    \centering
    \begin{minipage}{0.49\textwidth}
        \centering
        \begin{tabular}{l|l|l|l|l|l}
             \toprule
             CFG & FID$\downarrow$ & tFID$\downarrow$ & IS$\uparrow$ & Pre.$\uparrow$ & Rec.$\uparrow$ \\
             \midrule
             2.9 & \textbf{5.913} & \textbf{2.219} & 276.375 & 0.816 & \textbf{0.582}\\
             3.0 & 5.928 & 2.238 & 280.55 & 0.818 & 0.576\\
             3.1 & 5.933 & 2.240 & \textbf{284.50} & \textbf{0.819} & 0.572\\
             \bottomrule
        \end{tabular}
        \captionsetup{type=table}
        \caption{\textbf{Ablation of CMAR's CFG.} First place is \textbf{bolded}.}
        \label{tab:cmar-cfg}
    \end{minipage}
    \begin{minipage}{0.49\textwidth}
        \centering
        \begin{tabular}{c|c|c|c|c|c}
             \toprule
             Order & FID$\downarrow$ & tFID$\downarrow$ & IS$\uparrow$ & Pre.$\uparrow$ & Rec.$\uparrow$\\
             \midrule
             Ascend & 6.005 & 2.267 & 269.27 & 0.782 & 0.556\\
             Descend & \textbf{5.928} & \textbf{2.238} & \textbf{280.55} & \textbf{0.818} & \textbf{0.576}\\
             \bottomrule
        \end{tabular}
        \captionsetup{type=table}
        \caption{\textbf{Ablation of CMAR's synthesis order.} CMAR trained with descending collapse ranks has the best performance.}
        \label{tab:cmar-order}
    \end{minipage}
\end{figure*}

\begin{figure*}
    \centering
    \small
    \begin{tabular}{l|l|l|l|l|l|l|l|l|l}
         \toprule
         Method & T1@0\% & T5@0\% & T1@30\% & T1@30\% & T1@50\% & T5@50\% & T1@78\% & T5@78\% & AuC \\
         \midrule
         ViT & \textbf{82.91} & \textbf{96.28} & 80.03 & 94.80 & 74.38 & 91.48 & 22.38 & 36.76 & 57.16\\
         DynamicViT & 81.74 & 95.64 & \textbf{81.44} & \textbf{95.46} & \textbf{77.54} & \textbf{93.30} & 20.66 & 37.09 & 56.32\\
         ViT+C (Ours) & 82.84 & 96.23 & 78.09 & 93.82 & 71.67 & 89.58 & \textbf{31.04} & \textbf{49.45} & \textbf{57.27}\\
         \midrule
         RViT & 83.10 & 96.46 & 81.33 & 95.60 & 78.94 & 94.50 & 67.27 & 87.23 & 70.86\\
         CViT(Ours) & \textbf{83.11} & \textbf{96.50} & \textbf{81.37} & \textbf{95.69} & \textbf{79.39} & \textbf{94.63} & \textbf{70.57} & \textbf{88.94} & \textbf{72.19}\\
         \bottomrule
    \end{tabular}
    \captionsetup{type=table}
    \caption{\textbf{Classification accuracy of ViT under different settings.} First place is \textbf{bolded}. CViT achieves superior classification performance throughout mask rates. ViT+C degrades slightly without masking but outperforms ViT at 78\% masking.}
    \label{tab:vit-comp}
\end{figure*}

\subsection{Benchmarking CMAR}

We compare our collapse order's effect by quantitative metrics in \cref{tab:ar-comp}. MAR denotes the original model with random-order generation, and MAR+C is the same model inferred with collapse order. We measure generation performance by Fr\`{e}chet Inception Distance (tFID) \cite{FID} and Inception Score (IS) \cite{IS} following the original work \cite{mar} on 50,000 augmented training samples. Additionally, we measure the original FID, precision, and recall against the 10,000 samples in standard reference batch \cite{adm}.

Our CMAR achieves a significant $4\%$ gain in tFID despite of a very minor degradation in IS ($0.3\%$). The fine-tuning-less MAR+C achieves the highest IS score and has a tFID behind CMAR. Our two methods slightly degrades in metrics during generation without CFG, possibly due to their higher reliance on the conditioning from class-specific labels and orders which requires a stronger conditioning guidance provided by higher CFG scales.

The qualitative samples of MAR+C and CMAR are shown in \cref{fig:cmar-qual}. It can be observed that CMAR synthesizes slightly more realistic images than MAR+C, while there are significant artifacts in the baseline MAR such as object compositional defects or content confusion.

\subsection{Properties of CMAR}

We conduct an ablation of CFG scales for CMAR in \cref{tab:cmar-cfg}. A lower CFG achieves lower FIDs at the expense of IS, while higher CFG degrades in FIDs with an increase of IS. Therefore, we choose 3.0 as our optimal CFG scale to balance performance, different from the original setting at 2.9. Intuitively, this shift in CFG scale can be explained by CMAR's stronger reliance on conditioning labels in order to address diverse collapse orders in different classes.

Should CMAR follow ascending or descending patch ranks for generation? We train a separate model in each direction in \cref{tab:cmar-order} for ablation. The ascending model performs significantly worse than the descending one. This observation is consistent with the intuition that patches most relied upon by others should be generated first.

\subsection{Benchmarking CViT}

We compare CViT with four baselines under different patch mask ratios in \cref{tab:vit-comp}. The ViT model sees full images during training and is randomly masked during inference. ViT+C is a ViT variant that infers from collapse masks instead of random masking without modifying ViT's training. RViT is trained in the same way as CViT but with random masks. Although DynamicViT \cite{DynamicViT} is a token-pruning method that operates on model instead of data space, we include it to compare the effects of these separate pruning perspectives. Specifically, we employ the DeiT-B variant of DynamicViT which has similar parameter sizes as our ViT. We test these models' performance with top-k classification accuracy and an Area under Curve (AuC) metric that integrates over top-1 accuracies along the $0\sim99\%$ mask rates.

CViT leads in almost all metrics, demonstrating our collapse order's efficiency in capturing salient visual information from sparse high-rank patches. This effect is also observed in the training-less ViT+C, whose accuracy is higher than ViT and DynamicViT when extremely sparse 78\% patches are masked. Our AuC superiority under both scenarios suggests that the overall classification performance benefits from patch collapse modeling.

\subsection{Properties of CViT}
Our CViT maintains high classification accuracy with a significant portion of low-collapse-rank patches masked as shown in \cref{tab:vit-comp}. To analyze the accuracy decay, we find the knee of this concave accuracy curve with the Kneedle algorithm \cite{kneedle}. CViT's performance maintains until $78\%$ patches are omitted, at which point its accuracy drops to $70.6\%$. As we apply masking by dropping out patches from the input sequence instead of replacing them with mask tokens, the computational cost is reduced by $95.16\%$ from the $O\left(n^2\right)$ complexity of attention process. Furthermore, CViT also achieves the highest accuracy without masking. This result suggests that our patch collapse modeling can also benefit full-image classification.

It's also worth noticing that RViT, although following random order during training, also experiences a performance gain at lower mask ratios. However, its accuracy is lower than CViT at each level, suggesting our collapse order's superiority over stochastic modeling.

\section{Conclusion}
\label{sec:conclusion}

In this work, we introduce a novel modeling of the local feature uncertainty reduction process in images as patch collapse. By training a Collapse Masked Autoencoder (CoMAE) to reconstruct a target patch relying on other tiles, and analyzing the resultant patch dependency graph with PageRank, we are able to identify an optimal ordering of patches during image realization that maximally reduces uncertainty, \textit{i.e.}, the collapse order of patches. Experiments show that respecting this order benefits masked image modeling methods in: (1) autoregressive image generation, where the stochastic MIM generator MAR is boosted in FID and IS, and (2) image classification, which leads to a ViT that can maintain high accuracy despite seeing only 22\% image patches. As our computational resources are limited, we focus on learning only one variant of the CoMAE family and discuss future directions in Supplementary section 5. We hope our patch collapse modeling will encourage a new perspective on salient visual structures, benefiting future exploration on efficient and scalable vision methods.

\clearpage

{\LARGE \noindent \textbf{Supplementary}}
\setcounter{section}{0}

\section{Proof of Optimal Collapse Ranking}
\label{sec:proof}

We give a formal proof below to show that computing PageRank scores of the adjacency matrix $\mathbf{A}$ in Section 3.2 from CoMAE gives the optimal ranking $R$ in Eq. (1) that minimizes image uncertainty (cumulative patch entropy).

\begin{definition}[Cumulative conditional entropy]
    Define the cumulative conditional entropy of an ordered prefix $S$ by:
\begin{equation}
H_c(S) := \sum_{n=1}^N H\big(e_n \mid \{e_j : j \in S\}\big),
\end{equation}
where $H$ measures the distribution entropy.
\end{definition}

\begin{assumption}[Linear influence model]
    There exists $\beta \in \mathbb{R}_{\ge 0}^N$ such that the expected marginal reduction in $H_c$ by observing patch $j$ is approximately given by:
    \begin{equation}
        \Delta(j \mid S) \approx \big((I - cP)^{-1} \beta\big)_j.
    \end{equation}
\end{assumption}

\begin{assumption}[Submodularity]
    The entropy $H_c$ is monotone and approximately submodular in $S$.
\end{assumption}

\begin{theorem}[Optimal collapse ranking]
Let $\mathbf{A} \in [0,1]^{N \times N}$ be the learned dependency matrix with $\mathbf{A}_{ij}$ indicating the influence of patch $j$ on patch $i$. Let $P$ be the corresponding column-stochastic matrix, and let $c \in (0,1)$. Ordering patches in descending order of:
\begin{equation}
\label{eq:order}
    r := (I - cP)^{-1}\beta,
\end{equation}
minimizes the linearized proxy of $H_c$ at each prefix. If $\beta$ is constant or interpreted as a personalized teleport vector, this ordering corresponds to PageRank on $P$.
\end{theorem}

\begin{proof}
By the Neumann series expansion of \cref{eq:order}:
\begin{equation}
    r = (I - cP)^{-1}\beta = \sum_{t=0}^\infty c^t P^t \beta,
\end{equation}
we derive the following recurrence for $r$:
\begin{equation}
    cPr = cP\sum_{t=0}^\infty c^t P^t \beta 
= \sum_{t=0}^\infty c^{t+1} P^{t+1} \beta
= \sum_{s=1}^\infty c^s P^s \beta,
\end{equation}
where $s = t+1$. Therefore,
\begin{equation}
    \beta + cPr = \beta + \sum_{s=1}^\infty c^s P^s \beta = \sum_{s=0}^\infty c^s P^s \beta = r.
\end{equation}
Thus, $r$ satisfies the fixed-point equation:
\begin{equation}
    r = \beta + cPr.
\end{equation}

Under Assumption of Submodularity, greedy maximization of the marginal gain $\Delta(j \mid S)$ gives a $(1 - 1/e)$-approximation to the minimization of $H_c$ \cite{greedy}. Since the linear model assumes $\Delta(j \mid S) \approx r_j$, selecting nodes in descending order of $r_j$ yields an optimal prefix sequence for the cumulative entropy minimization objective.

Finally, note that the fixed-point equation above matches the PageRank formulation:
\begin{equation}
    r = (1-c)\beta + cPr,
\end{equation}
after normalizing $\beta$ so that it sums to one. Hence, $r$ is precisely the PageRank (or personalized PageRank) vector for $P$, completing the proof.
\end{proof}

\section{Additional Qualitative Results}
\label{sec:add-qual}

\begin{figure}
    \centering
    \includegraphics[width=\linewidth]{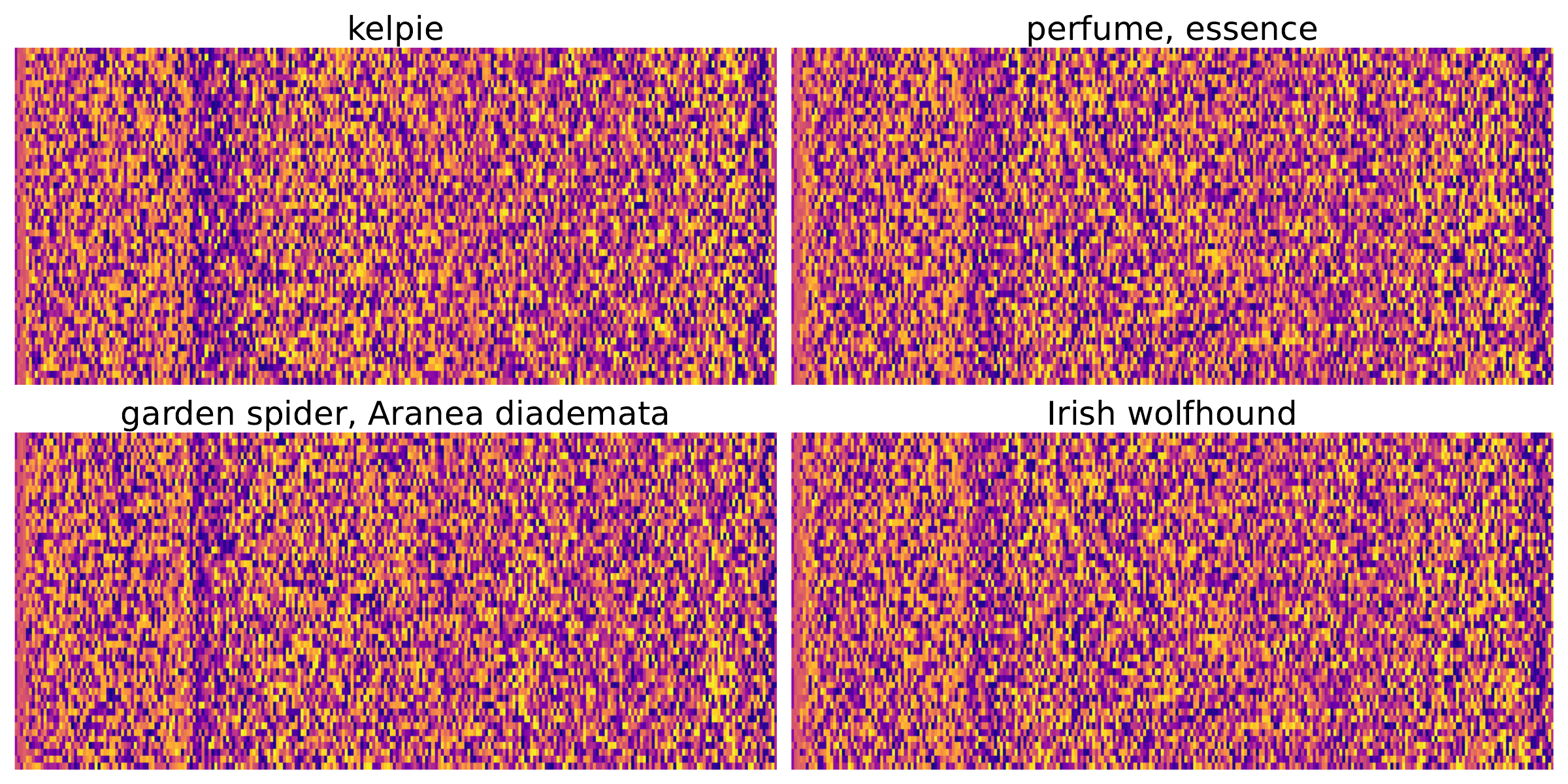}
    \caption{\textbf{Class-wise collapse order patterns.} These heatmaps show sample patch indices sorted in collapse order.}
    \label{fig:class-order}
\end{figure}

\begin{figure*}
    \centering
    \includegraphics[width=\linewidth]{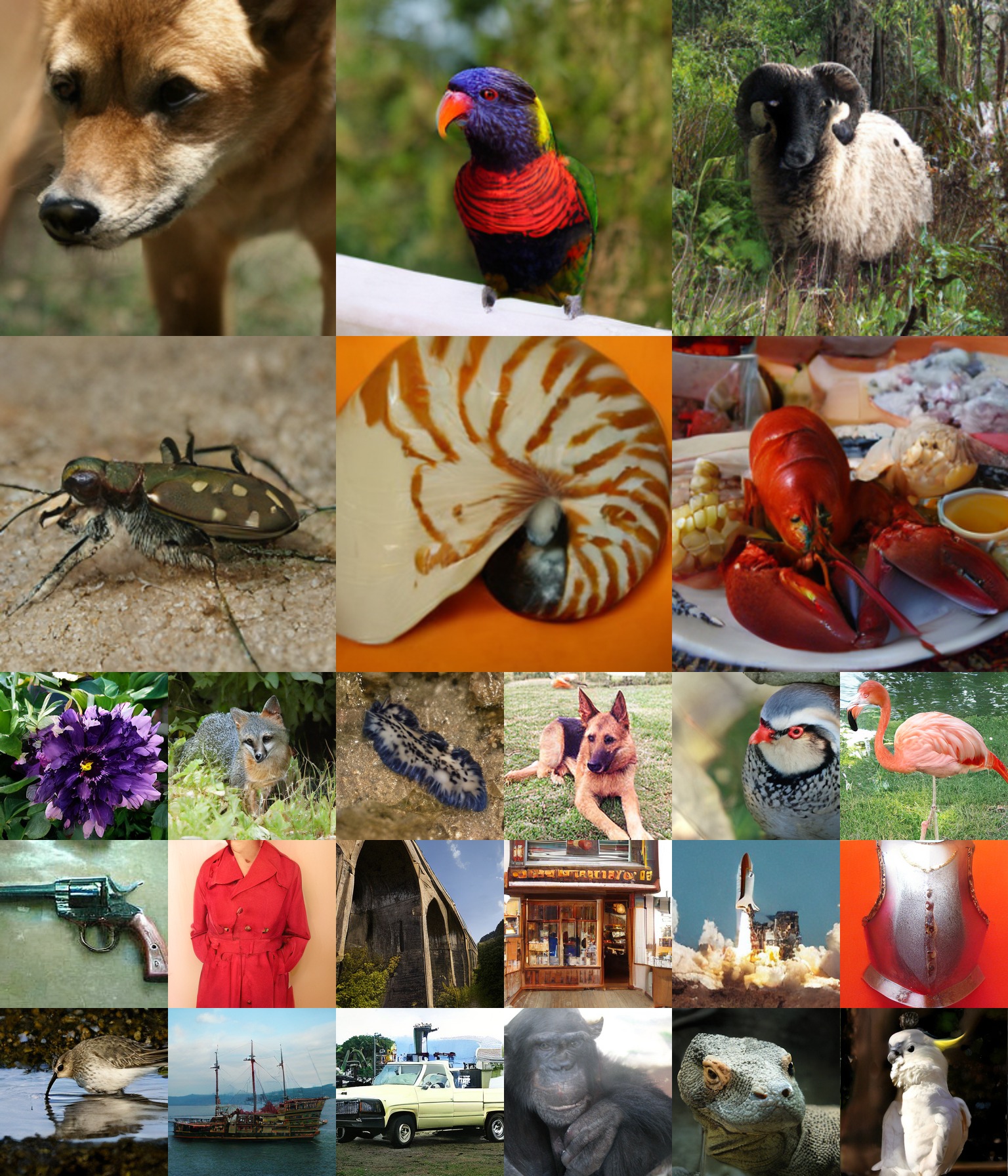}
    \caption{\textbf{Additional qualitative samples of CMAR.} Our method generates high-fidelity images across classes.}
    \label{fig:qual-sup}
    \vspace{12pt}
\end{figure*}

To analyze the identified collapse order beyond image-wise patterns, we visualize class-wise collapse order in \cref{fig:class-order}. It can be observed that similar collapse orders emerge among instances from the same class, since there are multiple closely aligned patch indices indicated by the vertical heat lines. The inter-class collapse order patterns are also similar judged by the locations of their heat lines. Therefore, it's reasonable to assume that the identified collapse order exhibit moderate consistency over classes and image instances, which suggests the existence of a common structure behind different images' realization.

Additonally, we provide more qualitative samples generated by our CMAR in \cref{fig:qual-sup}. It can be observed that our method synthesizes high-fidelity images across a broad range of classes.

\section{Model Architectures}
\label{sec:arch}

\begin{figure}
    \centering
    \includegraphics[width=\linewidth]{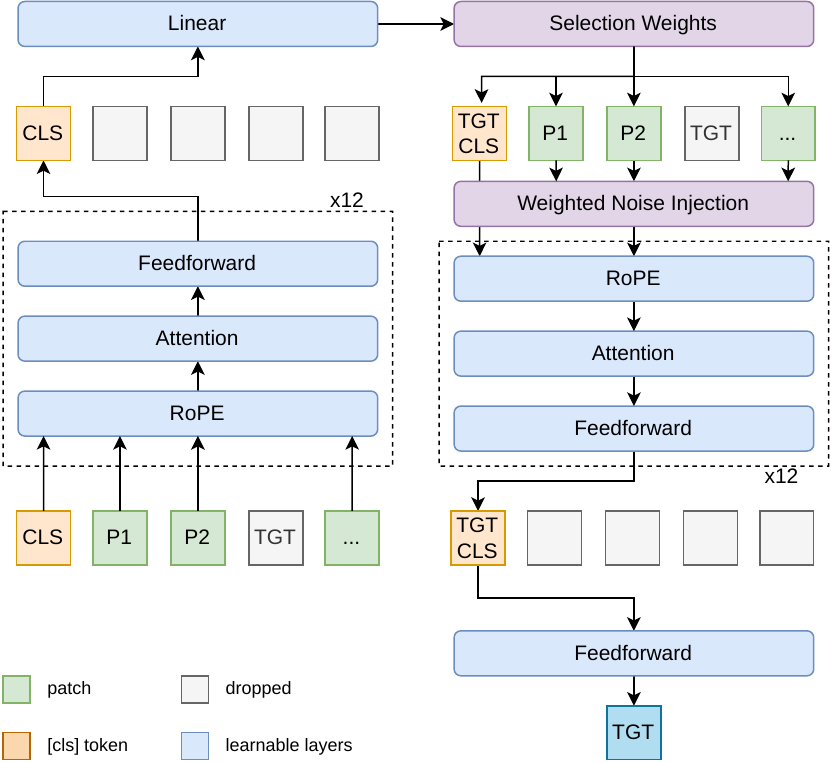}
    \caption{\textbf{Architecture of CoMAE.} Both the encoder and decoder follow ViT to pool information into a [cls] token.}
    \label{fig:arch}
\end{figure}

The architectures of MAR \cite{mar} and Vision Transformer \cite{ViT} are well-documented in their original papers. We elaborate on our novel CoMAE architecture in \cref{fig:arch} below.

Given the patch token sequence, the CoMAE encoder concatenates a learned [cls] token in front of it. The target patch is dropped to avoid information leakage during reconstruction. This sequence is then transformed by Rotary Position Embedding (RoPE) \cite{RoPE} to inform the encoder of each patch's relative image location. Next, we pass the sequence through 12 self-attention blocks with an embedding dimension of 256, retaining the processed [cls] token as output. Finally, a selection weight vector is obtained by passing this [cls] token through a linear projection, followed by a $\tanh$ normalization to keep each entry between $[0, 1]$.

We take the original input patch token sequence (without encoder processing) and inject them with Gaussian noise following Eq. (3) in Section 3.1. Again, the target patch is dropped for reconstruction. To inform the decoder which patch is currently under reconstruction, we append a [tgt cls] token in front of the entire sequence. The decoder then processes this sequence with 12 self-attention blocks. The embedding dimension is set as 64. Finally, we retrieve the first [tgt] token and pass it through a feedforward head to obtain the reconstructed target patch.

\section{Mask Polarization during Training}

\begin{figure}
    \centering
    \includegraphics[width=\linewidth]{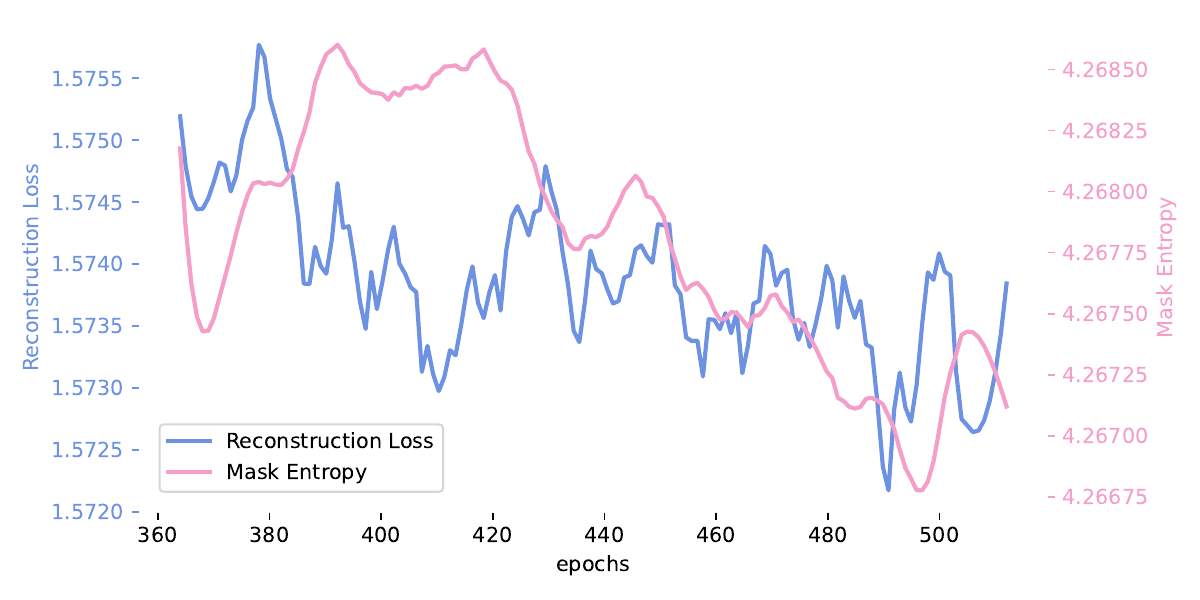}
    \caption{\textbf{Training of CoMAE (last 140 epochs).} Mask entropy drops together with reconstruction loss.}
    \label{fig:comae-train}
\end{figure}

As illustrated in \cref{fig:comae-train}, the masking entropy $H_\text{mask}$ progressively decreases as the reconstruction objective is optimized. This behavior indicates that CoMAE gradually learns to concentrate its masking distribution on a relatively small subset of patches that exert the greatest influence on the collapse of a given target patch during training. Rather than assigning equal importance to all contextual patches, the model identifies those that contain the most informative and causally relevant signals for reconstructing the target representation. CoMAE consequently imposes higher signal-to-noise ratio (SNR) weights on these informative patches, thereby preserving their content while allowing less relevant regions to be subjected to stronger corruption. This adaptive allocation of noise encourages the model to retain the contextual information that is most critical for preventing representational collapse and improving reconstruction quality.

\section{Limitations and Future Improvements}
\label{sec:limit}

As our collapse order describes local image units as patches, it's subject to the constraints of this representation: (1) each region has the same fixed size and (2) the shape of each patch doesn't reflect object saliency. In a more generalized setup, the image units can be expressed with salient local features such as segmentation maps or class activation maps. This extra layer of saliency fits more closely with our assumption of image locality during the collapse process. For now, we keep the modeling simple to study the preliminary feasibility of formalizing image collapse and leave these explorations to future works.

Additionally, we were unable to train more variants for CoMAE, CMAR, and CViT in this work due to limited computation. While we show that the current collapse order can already boost MIM methods in image generation and classification, scaling up training should achieve even higher performance gains. For instance, CoMAE can identify collapse orders more accurately with a deeper encoder, and larger MAR models can be converted into CMARs with longer training. More vision tasks can also be tested with collapse order for optimization, such as segmentation and detection. We will conduct these experiments if more computational resources become available in the future.

Finally, an alternative CoMAE design remains unexplored. Instead of training a decoder from scratch, we could directly employ MAR or similar autoregressive image generators for decoding. These pretrained decoders decouple encoder learning from the reconstruction objective, making it more efficient. Furthermore, CoMAE can also be deployed on the representation space of Representation Autoencoders (RAEs) \cite{RAE} to identify the process of \textbf{representation collapse} instead of patch collapse. We deem these directions as having high potentials in improving CoMAE and will study them soon.

\section{Ethical Statement}
\label{sec:ethics}

Our method studies the optimal scanning order of patches in image synthesis and classification. Since the identified high-rank patches in our collapse sequence have shown greater influence on these downstream tasks, visual patch attacks \cite{adpatch} could leverage them for more concentrated prompt engineering. However, the same information can also be adopted by visual models to enhance their robustness on high-rank patches to guard against such attacks. We will release our implementation code and model weights to aid the design of these defenses.

\clearpage

\bibliography{main}


\end{document}